\documentclass[11pt]{article}
\usepackage[margin=1in]{geometry}

\usepackage{amsmath,amssymb,amsfonts,amsthm}
\usepackage{booktabs}
\usepackage{enumitem}
\usepackage{hyperref}
\usepackage{multirow}
\usepackage[numbers,sort&compress]{natbib}
\usepackage{tikz-cd}
\usepackage{multicol}
\usepackage[ruled,lined,noend]{algorithm2e}

\newcommand{\R}{\mathbb{R}}
\newcommand{\x}{\mathbf{x}}
\newcommand{\h}{\mathbf{h}}

\newcommand{\ii}{\mathrm{i}}

\newtheorem{theorem}{Theorem}
\newtheorem{corollary}[theorem]{Corollary}
\newtheorem{remark}[theorem]{Remark}

\title{Nautile-370M: Spectral Memory Meets Attention in a Small Reasoning Model}
\author{Trickstr \\ \small Maixent Chenebaux}
\date{\today}

\begin{document}
\maketitle

\begin{abstract}
We present \textbf{Nautile-370M}, a 371-million-parameter small language model designed for efficient reasoning under strict parameter and inference budgets. Nautile-370M uses a hybrid backbone in which two \emph{SeqCond Attention} (SCA) layers---a linear-time spectral sequence operator inspired by SeqCondenser~\cite{chenebaux2024seqcondenser}---alternate with one transformer layer. This design aims to retain the long-context efficiency and state-tracking benefits of structured sequential models while preserving the expressive token-to-token routing of attention. The model was trained on a single Cloud TPU v4-64 pod slice provided through the Google TPU Research Cloud (TRC) program; the subsequent reinforcement learning stage was carried out on a single NVIDIA DGX Spark. We prove that the SCA readout mechanism can exactly retrieve any individual token from the prefix summary and can reproduce any output of softmax attention as a special case, establishing that SCA is at least as expressive as full self-attention in the continuous limit. We also describe the training data pipeline and outline a reinforcement learning stage specialized for reasoning, verification, and response quality.
\end{abstract}

\section{Introduction}

This paper describes \textbf{Nautile-370M}, a 371-million-parameter reasoning-oriented language model whose backbone alternates a novel sequence operator, \emph{SeqCond Attention} (SCA), with standard transformer layers. SCA computes a compressed summary of the prefix by evaluating the gradient of the empirical characteristic function at a small set of learned spectral points, and reads it out via a complex inner product. This gives it the efficiency of a linear recurrence ($O(1)$ state update at inference, parallel scan during training) while grounding the mechanism in a well-studied mathematical object.

The model was pretrained on a single TPU v4-64 pod slice (Google TRC) and post-trained with reinforcement learning on a single NVIDIA DGX Spark. During the RL stage we encountered a failure mode of standard GRPO on small models: when the policy's success rate is low, the negative-advantage gradient mass dominates and reasoning quality degrades. We propose two mitigations---a gradient-balanced GRPO variant and a scored self-distillation stage---that together bring GSM8K accuracy from 28.0\% to 33.4\%.

\paragraph{Contributions.}
\begin{itemize}[leftmargin=1.5em]
    \item \textbf{SCA (SeqCond Attention).} A sequence operator derived from the characteristic function of the prefix distribution. We provide the theoretical motivation and describe how it reduces to a trainable linear recurrence with complex-valued state (Section~2).
    \item \textbf{Theoretical expressiveness.} We prove that the SCA readout can extract any individual token from the prefix summary (Theorem~1), recover the full weighted distribution (Corollary~3), and reproduce any output of a softmax attention layer as a special case (Corollary~4). SCA is therefore at least as expressive as full self-attention in the continuous limit.
    \item \textbf{Hybrid architecture.} A 24-layer backbone interleaving 16 SCA layers and 8 transformer layers, totaling 371M parameters (Section~2).
    \item \textbf{Training pipeline.} A data curriculum combining 350B tokens of FineWeb-Edu with 250B tokens from SYNTH~\cite{pleias2024synth} and additional SYNTH-style synthetic chain-of-thought and instruction data distilled from multiple teacher models (Section~3).
    \item \textbf{Gradient-balanced GRPO.} A modification of standard GRPO that rescales the negative-advantage gradient to prevent it from dominating the positive one, enabling stable RL training at low success rates (Section~4).
    \item \textbf{Scored self-distillation.} An on-policy self-distillation stage that fine-tunes the model on its own verified correct traces, yielding an unexpected but noticeable boost in reasoning accuracy (Section~4).
\end{itemize}

\section{Model Architecture}

\subsection{Overview}

Nautile-370M is a decoder-only autoregressive language model with approximately 371 million parameters. Its backbone is organized into 24 layers arranged in repeated blocks of the form
\[
\text{SCA} \rightarrow \text{SCA} \rightarrow \text{Transformer},
\]
yielding 16 SCA layers and 8 transformer layers in total. The key architectural hyperparameters are summarized below.

\begin{center}
\begin{tabular}{ll}
\toprule
\textbf{Parameter} & \textbf{Value} \\
\midrule
Model dimension $D$ & 1024 \\
Feed-forward dimension $D_{\mathrm{ff}}$ & 2730 \\
Total layers & 24 (16 SCA + 8 Transformer) \\
Context length & 1024 \\
Tokenizer & \texttt{cl100k\_base} \\
Weight tying & Yes \\
\midrule
\multicolumn{2}{l}{\emph{SCA layers}} \\
Memory heads $K$ / Query heads $K'$ & 16 / 16 (no GQA) \\
Head dimension $H$ & 64 \\
Spectral samples $M$ & 2 \\
Expand factor & $2\times$ ($d_{\mathrm{inner}} = 2048$) \\
SwiGLU expansion & $3\times$ \\
Depthwise conv kernel & 4 \\
\midrule
\multicolumn{2}{l}{\emph{Transformer layers}} \\
Attention heads & 16 \\
KV heads (GQA) & 4 \\
Head dimension & 64 \\
\bottomrule
\end{tabular}
\end{center}

Each layer is wrapped with pre-norm RMS normalization~\cite{zhang2019rmsnorm} and residual connections. The SCA/SCA/Transformer ratio is motivated by the observation that much of language modeling consists of incremental state updates, while only a subset of tokens require global competitive selection. The transformer component is intentionally standard: a classical pre-norm causal transformer block~\cite{vaswani2017attention} with rotary position embeddings~\cite{su2021roformer} and grouped-query attention~\cite{ainslie2023gqa}, inserted periodically inside the hybrid stack.

\subsection{SeqCond Attention (SCA) Layer}

SCA is inspired by SeqCondenser~\cite{chenebaux2024seqcondenser}, a layer for inductive sequence representation, and adapts its characteristic-function mechanism to the autoregressive generation setting. It is derived from a principled theoretical object---the derivative of a characteristic function---and then systematically discretized into a practical neural layer. We present the full derivation before describing the implementation.

\subsubsection{Theoretical Foundation: Characteristic Prefix Summary}

Let $(\h_1,\dots,\h_t)$ be the sequence of embeddings observed up to step~$t$. We model this prefix as the support of a discrete random variable $X$ in the embedding space and summarize it through the \emph{characteristic function}
\[
\varphi_X(\theta) \;=\; \mathbb{E}\big[e^{\ii \langle \theta, X \rangle}\big], \qquad \theta \in \R^d.
\]
By the injectivity theorem, $\varphi_X$ determines the distribution of~$X$ uniquely. All moments, when they exist, are recoverable:
\[
\mathbb{E}[X^{\otimes n}] \;=\; (-\ii)^n \,\nabla_\theta^n \varphi_X(\theta)\big|_{\theta=0}.
\]
The characteristic representation is therefore a \emph{sufficient statistic} for the sequence-induced distribution: no information is lost.

SCA operates not on $\varphi_X$ itself but on its \emph{gradient} with respect to the spectral variable:
\[
\nabla_\theta \varphi_X(\theta) \;=\; \ii\,\mathbb{E}\big[X\, e^{\ii \langle \theta, X \rangle}\big].
\]
The signal $X$ now appears \emph{multiplicatively} inside the expectation: the derivative summary jointly encodes the spectral phase \emph{and} the values of the underlying random variable. It is therefore strictly more informative for downstream readout than $\varphi_X$ alone, which only carries phase. This gradient object is the core representation maintained by SCA.

\subsubsection{Spectral Readout via Hermitian Inner Product}

Given the derivative summary $S_{1:t}(\theta) = \nabla_\theta \varphi_{X_{1:t}}(\theta)$, the layer extracts a token-conditioned output by computing a \emph{Hermitian inner product} in the function space $L^2(\R^d, \mathbb{C})$. A \emph{spectral query} $w_t(\theta)$ is derived from the current token, and the readout is defined as
\[
o_t \;=\; \langle S_{1:t},\, w_t \rangle_{L^2} \;=\; \int S_{1:t}(\theta)\,\overline{w_t(\theta)}\, d\theta.
\]
This Hermitian inner product plays the role of an attention mechanism: the current token emits a query and the prefix summary returns a value. The crucial difference is that the memory being queried is a spectral condensation of the entire prefix, not a stored list of past key-value vectors. The conjugation $\overline{w_t}$ ensures sesquilinearity, which is the natural pairing for complex-valued functions and preserves phase-sensitive retrieval.

The readout is not merely a heuristic analogy with attention: it is provably capable of exact individual token retrieval.

\begin{theorem}[Exact token retrieval]
Let $\h_1, \ldots, \h_t \in \R^d$ be pairwise distinct embeddings. For every index $j \in \{1,\ldots,t\}$, the spectral query
\[
w_j(\theta) \;=\; \frac{\ii\, t}{(2\pi)^d}\, e^{\ii \langle \theta,\, \h_j \rangle}
\]
satisfies
\[
\int_{\R^d} S_{1:t}(\theta)\;\overline{w_j(\theta)}\; d\theta \;=\; \h_j,
\]
where the integral is interpreted in the distributional (Fourier-inversion) sense.
\end{theorem}

\begin{proof}
Expanding the summary and the conjugated query:
\begin{align*}
\int_{\R^d} S_{1:t}(\theta)\;\overline{w_j(\theta)}\; d\theta
&= \int_{\R^d} \frac{\ii}{t}\sum_{k=1}^{t} \h_k\, e^{\ii\langle\theta,\h_k\rangle}\;\cdot\; \frac{-\ii\, t}{(2\pi)^d}\, e^{-\ii\langle\theta,\h_j\rangle}\; d\theta \\[4pt]
&= \frac{1}{(2\pi)^d}\sum_{k=1}^{t} \h_k \int_{\R^d} e^{\ii\langle\theta,\,\h_k - \h_j\rangle}\; d\theta \\[4pt]
&= \frac{1}{(2\pi)^d}\sum_{k=1}^{t} \h_k \;\cdot\; (2\pi)^d\,\delta(\h_k - \h_j) \;=\; \h_j,
\end{align*}
where the penultimate equality uses the identity $\int_{\R^d} e^{\ii\langle\theta,\xi\rangle}\,d\theta = (2\pi)^d\,\delta(\xi)$.
\end{proof}

\begin{remark}
The theorem establishes that the derivative characteristic summary is a \emph{lossless} encoding of the prefix: every individual token can be recovered by choosing the appropriate spectral query. Any gap between this theoretical guarantee and the practical behavior of the layer stems from the approximations introduced when discretizing the continuous theory---not from the mechanism itself. Those approximations are described in Section~2.2.3. We formalize the connection with self-attention in Corollary~4 below.
\end{remark}

The result extends to the recovery of the full weighted distribution. In the general (non-uniform) setting, the empirical characteristic function is $\varphi_X(\theta) = \sum_{k=1}^t p_k\, e^{\ii\langle\theta,\h_k\rangle}$ with arbitrary weights $p_k > 0$, $\sum_k p_k = 1$.

\begin{corollary}[Full distribution recovery]
Let $\h_1,\ldots,\h_t \in \R^d$ be pairwise distinct with associated weights $p_1,\ldots,p_t > 0$. Then:
\begin{enumerate}[leftmargin=1.5em]
    \item The derivative summary $S_{1:t}(\theta) = \ii\sum_{k=1}^{t} p_k\,\h_k\, e^{\ii\langle\theta,\h_k\rangle}$ encodes the \emph{weighted} embeddings: the readout with query $w_j(\theta) = \frac{1}{\ii\,(2\pi)^d}\, e^{\ii\langle\theta,\h_j\rangle}$ yields
    \[
    \int_{\R^d} S_{1:t}(\theta)\;\overline{w_j(\theta)}\; d\theta \;=\; p_j\,\h_j.
    \]
    \item The weights alone are recoverable from $\varphi_X$: the scalar query $\tilde{w}_j(\theta) = \frac{1}{(2\pi)^d}\, e^{\ii\langle\theta,\h_j\rangle}$ gives
    \[
    \int_{\R^d} \varphi_X(\theta)\;\overline{\tilde{w}_j(\theta)}\; d\theta \;=\; p_j.
    \]
    \item Since $S_{1:t} = \nabla_\theta \varphi_X$ determines $\varphi_X$ up to the known constant $\varphi_X(0) = 1$, the gradient summary alone suffices to recover both the support $\{\h_k\}$ and the weights $\{p_k\}$.
\end{enumerate}
\end{corollary}

\begin{proof}
Part~(1) follows from Theorem~1 with the substitution $\h_k \mapsto p_k\,\h_k$ in the Fourier inversion step. Part~(2) is the same argument applied to the scalar function $\varphi_X$. For part~(3), observe that $\varphi_X(\theta) = 1 + \int_0^\theta S_{1:t}(\tau)\cdot d\tau$ (path integral from the origin), so $\varphi_X$ is determined by $S_{1:t}$; combining (1) and (2) then recovers $p_j$ and $\h_j$ separately.
\end{proof}

The linearity of the integral immediately yields the connection with attention:

\begin{corollary}[SCA subsumes self-attention]
For any coefficients $\alpha_1,\ldots,\alpha_t \in \R$ (not necessarily non-negative or summing to one), the composite spectral query $w(\theta) = \sum_{k=1}^{t} \alpha_k\, w_k(\theta)$ satisfies
\[
\int_{\R^d} S_{1:t}(\theta)\;\overline{w(\theta)}\; d\theta \;=\; \sum_{k=1}^{t} \alpha_k\, \h_k.
\]
In particular, any output of a softmax attention layer, which computes a convex combination $\sum_k \alpha_k V_k$ with $\alpha_k \geq 0$, $\sum_k \alpha_k = 1$, is a special case of the SCA readout. The spectral mechanism is therefore \emph{at least as expressive as full self-attention} in the continuous limit, and strictly more general since it places no non-negativity or normalization constraint on the retrieval weights.
\end{corollary}

\begin{proof}
By linearity of the integral and Theorem~1:
$\int S_{1:t}(\theta)\,\overline{w(\theta)}\,d\theta = \sum_k \alpha_k \int S_{1:t}(\theta)\,\overline{w_k(\theta)}\,d\theta = \sum_k \alpha_k\,\h_k.$
\end{proof}

\begin{remark}[Why the gradient, not the function]
The choice to operate on $\nabla_\theta\varphi_X$ rather than $\varphi_X$ is not cosmetic---it is structurally necessary for value retrieval. Since $\varphi_X(\theta)$ is \emph{scalar}-valued, any inner-product readout against a scalar query can only produce a scalar: applying Fourier inversion to $\varphi_X$ recovers the weight $p_j$ of token~$j$, but never the embedding $\h_j \in \R^d$ itself. The information about $\{\h_k\}$ is in principle contained in $\varphi_X$ (by the injectivity theorem), but extracting it requires solving a spectral estimation inverse problem.

The gradient $\nabla_\theta\varphi_X(\theta) \in \mathbb{C}^d$ is \emph{vector}-valued, and each embedding $\h_k$ appears as a multiplicative coefficient of its exponential. This turns the inverse problem into a direct, single-step linear readout (Theorem~1). In short: the gradient transforms a hard retrieval problem into a trivial one.
\end{remark}

Everything up to this point is \emph{exact}: the characteristic function, its gradient, the $L^2$ inner product, and the retrieval theorems are objects of pure mathematics, with no approximation, no parameterization, and no finite-dimensional constraint. The next section introduces the compromises required to turn this ideal formulation into a trainable neural layer.

\subsubsection{From Continuous Theory to Discrete Layer}

We now bridge the gap between the ideal objects above and the practical implementation. The only approximation required is the discretization of the continuous spectral domain $\theta \in \R^d$ to a finite set of evaluation points. Note that the prefix distribution itself was always discrete---a sum over the observed embeddings $\h_1,\ldots,\h_t$---so no approximation is introduced on that side.

\paragraph{Finite spectral grid.}
The readout integral $o_t = \int S_{1:t}(\theta)\,\overline{w_t(\theta)}\,d\theta$ ranges over all of $\R^d$, which is not directly computable. We replace it by a \emph{learned quadrature rule}: a finite set of $M$ spectral evaluation points $\{\theta_m\}_{m=1}^M$ with associated weights $\{\omega_m\}$, both trained end-to-end. This is analogous to approximating a Fourier integral by a weighted sum over a discrete set of frequencies---the classical setting of numerical quadrature, except that here the nodes and weights are optimized by gradient descent rather than fixed by a deterministic rule. With $M$ small (in our case $M=2$), the approximation is coarse but sufficient because the downstream loss directly supervises which spectral regions are useful. Note that $M$ counts spectral points \emph{per head}: for Nautile-370M with $K=16$ memory heads, the total number of learned spectral points is $K \times M = 32$, which is considerably larger than the per-head count suggests.
\[
o_t = \int S_{1:t}(\theta)\,\overline{w_t(\theta)}\, d\theta \;\approx\; \sum_{m=1}^{M} \omega_m\; S_{1:t}(\theta_m)\;\overline{w_t(\theta_m)}.
\]

\paragraph{Causal summary.}
The theoretical derivative summary $S_{1:t}(\theta) = \ii\sum_{k=1}^{t}\frac{1}{t}\,\h_k\,e^{\ii\langle\theta,\h_k\rangle}$ is already a finite sum over the discrete prefix. For the practical layer, we introduce two additional degrees of freedom: unnormalized positive contribution weights $\alpha_\tau$ and a learned feature map $\psi$ inside the phase. These are parameterization choices, not approximations. Evaluated at the discrete spectral grid, the summary becomes a running sum:
\[
S_t(\theta_m) \;=\; \sum_{\tau=1}^{t} \alpha_\tau\; \h_\tau\; \exp\!\big(\ii\,\theta_m \cdot \psi(\h_\tau)\big),
\]
where $\alpha_\tau > 0$ is an unnormalized contribution weight and $\psi$ is a learned feature map. Unlike a normalized probabilistic formulation (where $\sum_\tau \alpha_\tau = 1$), the weights are kept \emph{positive but unconstrained}, which increases expressivity and avoids a costly normalization pass.

Because this summary is additive in~$\tau$, causality is immediate: $S_t = S_{t-1} + \alpha_t\,\h_t\,e^{\ii\theta_m \cdot \psi(\h_t)}$. This additive structure is what enables the prefix-scan implementation used in training and constant-time updates at inference.

\subsubsection{Implementation}

The following describes the concrete forward pass. The layer has $K$ memory heads, $K'$ query heads (with GQA ratio $K/K'$), head dimension $H$, and $M$ spectral sample points. Let $\x_t \in \R^D$ be the input at position~$t$.

\paragraph{Step 1: Input projection \& local mixing.}
A single dense projection followed by a causal depthwise convolution (kernel size~$c$) and SiLU activation produces two branches:
\[
[\,\mathbf{z}_{\mathrm{mem}}\;;\;\mathbf{z}_{\mathrm{query}}\,] \;=\; \operatorname{SiLU}\!\big(\operatorname{DWConv}_c(W_{\mathrm{in}}\,\x_t)\big).
\]
The memory branch $\mathbf{z}_{\mathrm{mem}}$ yields per-head key values $\mathbf{k}_t \in \R^{K \times H}$ and a scalar score $s_t \in \R^K$. The query branch $\mathbf{z}_{\mathrm{query}}$ yields spectral query coordinates $q_t^{\mathrm{re}}, q_t^{\mathrm{im}} \in \R^{K' \times H \times M}$.

\paragraph{Step 2: Contribution weights.}
The positive contribution weight decomposes into a content gate and a temporal decay:
\[
\alpha_t \;=\; \underbrace{\operatorname{softplus}\!\big(\gamma\, s_t + \beta\big)}_{\text{content gate}} \;\times\; \underbrace{\exp\!\big(-\lambda_k \cdot d(t)\big)}_{\text{temporal decay}},
\]
where $\gamma, \beta$ are per-head learned scale and bias, $\lambda_k > 0$ is a per-head decay slope (parameterized via softplus in log-space), and $d(t)$ is a distance-to-boundary function.

\paragraph{Step 3: Phase modulation \& complex encoding.}
The phase is computed via a bounded softsign modulation:
\[
\phi_t \;=\; \frac{\eta\,\mathbf{k}_t}{1 + |\eta\,\mathbf{k}_t|} \;\odot\; \boldsymbol{\theta}, \qquad
\mathbf{r}_t + \ii\,\mathbf{i}_t \;=\; \alpha_t\,\mathbf{k}_t\,e^{\ii\,\phi_t},
\]
where $\eta$ is a per-head learned phase scale and $\boldsymbol{\theta} \in \R^{K \times H \times M}$ is the learned spectral grid.

\paragraph{Step 4: Causal accumulation.}
Real and imaginary parts are accumulated by causal prefix sum (or matrix multiply for short sequences):
\[
R_t = \sum_{\tau=1}^t \mathbf{r}_\tau, \quad
I_t = \sum_{\tau=1}^t \mathbf{i}_\tau, \quad
Z_t = \sum_{\tau=1}^t \alpha_\tau.
\]
The normalized state is $\hat{R}_t = R_t / Z_t$, $\hat{I}_t = I_t / Z_t$.

\paragraph{Step 5: Spectral readout.}
The Hermitian match between state and query, scaled by $1/\sqrt{H}$, is integrated over spectral samples:
\[
o_t^{\mathrm{re}} = \sum_{m=1}^M \omega_m \big(\hat{R}_t\, q_t^{\mathrm{re}} + \hat{I}_t\, q_t^{\mathrm{im}}\big)_m, \qquad
o_t^{\mathrm{im}} = \sum_{m=1}^M \omega_m \big(\hat{I}_t\, q_t^{\mathrm{re}} - \hat{R}_t\, q_t^{\mathrm{im}}\big)_m.
\]

\paragraph{Step 6: Output fusion.}
The concatenated complex output $[\,o_t^{\mathrm{re}}\;;\;o_t^{\mathrm{im}}\,]$ passes through gated RMS normalization (gated by a projection of the original input~$\x_t$), a per-head SwiGLU~\cite{shazeer2020glu} expansion, and a final dense projection back to~$\R^D$.

Algorithm~\ref{alg:sca} summarizes the complete forward pass.

\begin{algorithm}[t]
\caption{SCA Layer --- Forward Pass}\label{alg:sca}
\SetKwInOut{Input}{Input}
\SetKwInOut{Output}{Output}
\SetKwInOut{Params}{Params}
\Input{$X \in \R^{B \times L \times D}$}
\Output{$Y \in \R^{B \times L \times D}$}
\Params{$W_{\mathrm{in}}, W_{\mathrm{out}}, W_{\mathrm{gate}}, W_{\mathrm{read}}$; spectral grid $\boldsymbol{\theta} \in \R^{K \times H \times M}$; integration weights $\omega \in \R^{K' \times H \times M}$; decay slopes $\lambda$; phase scale $\eta$; score scale $\gamma$, bias $\beta$}
\BlankLine
\tcp{Step 1: Input projection \& local mixing}
$[\,\mathbf{z}_{\mathrm{mem}}\;;\;\mathbf{z}_{\mathrm{q}}\,] \gets \mathrm{SiLU}\!\big(\mathrm{DWConv}(W_{\mathrm{in}}\, X)\big)$\;
$\mathbf{k} \gets \mathbf{z}_{\mathrm{mem}}[\text{:dim\_mem}].\ \mathrm{reshape}(B,L,K,H)$\;
$s \gets \mathbf{z}_{\mathrm{mem}}[\text{dim\_mem:}]$\;
$q^{\mathrm{re}}, q^{\mathrm{im}} \gets \mathrm{split}\!\big(\mathbf{z}_{\mathrm{q}}.\ \mathrm{reshape}(B,L,K',H,M,2)\big)$\;
\BlankLine
\tcp{Step 2: Contribution weight}
$\alpha \gets \mathrm{softplus}(\gamma \cdot s + \beta)\;\times\;\exp(-\lambda \cdot d(t))$\;
\BlankLine
\tcp{Step 3: Phase modulation \& complex encoding}
$\phi \gets \mathrm{softsign}(\eta \cdot \mathbf{k})\;\odot\;\boldsymbol{\theta}$\;
$\mathbf{r} + \ii\,\mathbf{i} \gets \alpha\,\mathbf{k}\,e^{\ii\,\phi}$\;
\BlankLine
\tcp{Step 4: Causal accumulation (prefix sum)}
$R_t \gets \mathrm{cumsum}(\mathbf{r})$\;
$I_t \gets \mathrm{cumsum}(\mathbf{i})$\;
$Z_t \gets \mathrm{cumsum}(\alpha)$\;
$\hat{R}_t \gets R_t / Z_t$\;
$\hat{I}_t \gets I_t / Z_t$\;
\BlankLine
\tcp{Step 5: Spectral readout (Hermitian inner product)}
$o^{\mathrm{re}} \gets \sum_{m} \omega_m (\hat{R}_t\, q^{\mathrm{re}} + \hat{I}_t\, q^{\mathrm{im}})_m \;/\; \sqrt{H}$\;
$o^{\mathrm{im}} \gets \sum_{m} \omega_m (\hat{I}_t\, q^{\mathrm{re}} - \hat{R}_t\, q^{\mathrm{im}})_m \;/\; \sqrt{H}$\;
\BlankLine
\tcp{Step 6: Output fusion}
$\mathbf{o} \gets \mathrm{GatedRMSNorm}\!\big([\,o^{\mathrm{re}};\,o^{\mathrm{im}}\,],\; W_{\mathrm{gate}}\,X\big)$\;
$Y \gets W_{\mathrm{out}}\;\mathrm{SwiGLU}(W_{\mathrm{read}}\;\mathbf{o})$\;
\end{algorithm}

\subsubsection{Properties}

Compared with full self-attention, SCA offers three structural advantages:
\begin{enumerate}[leftmargin=1.5em]
    \item \textbf{Lossless prefix encoding (in theory):} the derivative characteristic summary preserves the full statistical content of the prefix distribution; the practical layer is a finite-dimensional approximation of this object.
    \item \textbf{Linear-time causal accumulation:} the additive state update yields $O(L)$ prefix scans during training and $O(1)$ state updates at decoding.
    \item \textbf{Structured conditional retrieval:} extraction is a learned Hermitian inner product in spectral space, not a dense pairwise attention over the full prefix.
\end{enumerate}

\subsection{Hybrid Layer Design}

Let $f_s$ denote a SCA layer and $f_a$ a transformer layer. The backbone repeats a $f_s \circ f_s \circ f_a$ motif:
\[
F = (\,f_a \circ f_s \circ f_s\,)^{L/3}.
\]

\paragraph{Why not SCA alone?}
Corollary~4 shows that SCA subsumes self-attention in the continuous limit. However, the practical layer evaluates the spectral integral at only $M=2$ points per head. This finite quadrature cannot reproduce an arbitrary convex combination over $t$ prefix tokens when $t \gg K \cdot M$. Operations that require precise pairwise comparison---such as coreference resolution, exact copying, or symbolic alignment---may therefore exceed the effective bandwidth of a single discretized SCA layer. Periodic attention layers provide an exact $O(L^2)$ fallback for these operations.

\paragraph{Complementary computational primitives.}
SCA and attention differ not only in cost but in \emph{kind}. SCA maintains a fixed-size complex state that is updated additively at each position ($O(1)$ per step); it excels at incremental accumulation of distributional statistics over the prefix. Attention computes explicit pairwise scores across all positions ($O(L^2)$ per layer); it excels at selective routing where a small number of specific tokens must be compared. The two mechanisms are therefore structurally complementary: SCA handles the bulk of context propagation, and attention handles the sparse precise comparisons that a coarse spectral summary cannot resolve.

\paragraph{The 2:1 ratio.}
The choice of two SCA layers per attention layer was guided by the literature on hybrid SSM-Transformer architectures~\cite{gu2023mamba,lieber2024jamba,gu2022s4} rather than by systematic ablation. With a fixed 30-day TPU allocation and no guarantee that the architecture would be competitive at this scale, we had to commit to a configuration early and train to convergence. The 2:1 ratio was a pragmatic bet: it allocates roughly two-thirds of the depth to efficient $O(1)$ state propagation and one-third to exact token-to-token routing, consistent with the ratios reported in prior hybrid work. When intermediate checkpoints suggested suboptimal behavior, we applied marginal corrections (learning rate adjustments, data mixing), but the layer ratio itself was never revised. Whether a different split (e.g.\ 3:1 or 1:1) would improve performance at this scale remains an open question that we did not have the compute budget to explore.

\section{Training Data}

\subsection{Curriculum}

Training proceeds in two stages:
\begin{enumerate}[leftmargin=1.5em]
    \item \textbf{FineWeb-Edu~\cite{penedo2024fineweb} (\texttildelow350B tokens):} broad factual and linguistic coverage---entities, discourse structure, expository text---providing the knowledge base for downstream reasoning.
    \item \textbf{SYNTH~\cite{pleias2024synth} (\texttildelow250B tokens):} PleIAs' large-scale synthetic reasoning corpus, emphasizing explicit chain-of-thought traces, structured answers, and instruction following. This stage converts the latent knowledge acquired in stage~1 into operational reasoning behavior.
\end{enumerate}

\subsection{Synthetic Augmentation and Template Distillation}

On top of the main curriculum, we add approximately 4 million synthetic documents aligned to the SYNTH~\cite{pleias2024synth} template. Their role is to improve instruction following, response formatting, and the practical elicitation of knowledge already stored in the model. These documents are generated from diverse instruction, conversational, creative-writing, and assistant datasets, but are rewritten to match the style and structure of the PleIAs SYNTH corpus as closely as possible.

To keep this augmentation on-template, we use retrieval-guided generation over a sample of SYNTH~\cite{pleias2024synth} itself. For each prompt, we retrieve the five nearest SYNTH examples and inject them into the teacher prompt. The teacher then produces a new answer conditioned on nearby in-template examples. This is effectively a form of distillation with format guidance~\cite{hinton2015distilling}.

\subsection{Teacher Models and Difficulty-Aware Distillation}

The synthetic augmentation is distilled from a mixture of teacher models, including GPT-OSS-20B, GPT-OSS-120B, Mistral Small 3.2, and Mistral Large 3, with teacher choice depending on task difficulty. In practice, these few million documents are critical: they teach the model to follow instructions, respect answer formats, and make better use of knowledge that is already present in the weights.

\subsection{Compute Infrastructure}

All pretraining and supervised fine-tuning was performed on a single Cloud TPU v4-64 pod slice (64 TPU v4 chips) provided through the \emph{Google TPU Research Cloud} (TRC) program, a research initiative that grants temporary access to Cloud TPU resources at no cost. The JAX-based training stack was designed to run entirely within this allocation. After the 30-day TRC allocation expired, the reinforcement learning stage (Section~4) was carried out on a single NVIDIA DGX Spark.

\section{Reinforcement Learning for Reasoning}

\subsection{Motivation}

Because the SYNTH~\cite{pleias2024synth} corpus was distilled from a mixture of teacher models (Section~3), the supervised model produces chain-of-thought (CoT)~\cite{wei2022cot} traces that vary in style, verbosity, and structure. The reasoning is functional but lacks unity. The reinforcement learning stage addresses this coherence gap through a three-stage pipeline: LLM-judge--based GRPO for format alignment, a gradient-balanced GRPO variant for reasoning, and on-policy self-distillation.

\paragraph{Standard GRPO on Nautile-370M.}
A natural first attempt is to use verifiable rewards on mathematical benchmarks (e.g.\ exact-match correctness on GSM8K) and train with standard GRPO~\cite{shao2024deepseekmath}. We tried this extensively from the SFT checkpoint---varying compute dtypes, clipping schedules, learning rates, and reward formulations including soft/continuous variants---but were unable to produce stable improvements. The proximate cause is a gradient imbalance: with a 28\% GSM8K success rate, roughly 72\% of sampled completions receive negative advantages while only~28\% receive positive ones, so the gradient ascent term (suppressing incorrect traces) systematically overwhelms the gradient descent term (reinforcing correct ones).

We hypothesize that this imbalance is exacerbated by the nature of our pretraining. Because Nautile-370M was trained on~250B tokens of chain-of-thought data (SYNTH), reasoning is not a behavior acquired ad hoc during RL---it is an integral part of the model's representations. As a result, most incorrect completions follow a structurally sound chain of thought but arrive at the wrong answer due to knowledge gaps (factual errors, arithmetic slips) rather than reasoning failures. In our case, standard GRPO treated these traces as uniformly bad and actively suppressed them, degrading the very reasoning capability that pretraining built in. We do not claim this generalizes beyond our specific setup; it may be an artifact of the unusually high proportion of chain-of-thought SFT data in our pretraining.

\subsection{Stage 1: Format Alignment via Dr.\ GRPO}

We run 1\,200 steps of \emph{Dr.\ GRPO}~\cite{liu2025drgrpo}, a GRPO variant that drops standard-deviation normalization of advantages: the group-relative advantage of completion $i$ is simply $A_i = r_i - \bar{r}$, where $\bar{r}$ is the group mean, without dividing by the within-group standard deviation. This prevents the gradient signal from being inflated by trivially easy or trivially hard groups.

\paragraph{Reward design.}
Each completion in a group of four candidates is scored by \textbf{Mistral Large~3} acting as an LLM judge. The judge produces three criterion scores on a 1--5 scale and a holistic score from 0--100. The scalar reward is
\[
r = 0.5 \times \underbrace{\frac{0.30\, s_{\text{reason}} + 0.55\, s_{\text{answer}} + 0.15\, s_{\text{follow}}}{5}}_{\text{weighted criterion score}} + 0.5 \times \frac{s_{\text{overall}}}{100},
\]
with an additive overlong penalty applied to completions that exceed the generation budget. Groups in which every completion receives an overall score above~90 (and a minimum above~85) are skipped as mastered.

\paragraph{Update rule.}
The policy gradient loss uses \emph{token-level normalization} (DAPO-style~\cite{liu2025drgrpo}): the loss for each completion is weighted by its token count divided by the total tokens in the group, so every token contributes equally regardless of response length. A KL penalty against a frozen reference model (the SFT checkpoint) is added with coefficient $\beta$.

This stage successfully unifies the output format: responses become cleaner and more consistent in structure. However, it does not produce measurable gains on downstream reasoning benchmarks (+0.98 pp on GSM8K, see Table~\ref{tab:rl_gsm8k}), confirming that format alignment and reasoning improvement are largely orthogonal objectives at this scale.

\subsection{Stage 2: Gradient-Balanced GRPO}

With the output format stabilized by Stage~1, we switch to a verifiable reward signal on GSM8K. We first re-attempted standard GRPO from the Stage~1 checkpoint, hoping that the improved formatting would raise the success rate enough to escape the gradient imbalance described in Section~4.1. It did not: standard GRPO produced the same degradation pattern as before the RLAI stage, confirming that the imbalance is not an artifact of hyperparameter choice or starting checkpoint.

To correct this, we decouple the positive and negative components of the policy gradient and normalize their magnitudes independently:
\[
g^+ = \sum_{i:\, A_i > 0} A_i \,\nabla_\theta \log \pi_\theta(y_i \mid x), \qquad
g^- = \sum_{i:\, A_i < 0} A_i \,\nabla_\theta \log \pi_\theta(y_i \mid x),
\]
\[
g = g^+ + \frac{\|g^+\|}{\|g^-\| + \epsilon}\, g^-.
\]
This ensures that the destructive gradient component never dominates the constructive one, regardless of the proportion of correct completions in the group. With this modification, we obtain consistent improvements when training on GSM8K~\cite{cobbe2021gsm8k} (+2.40 pp over Stage~1, see Table~\ref{tab:rl_gsm8k}). However, the gains plateau after roughly 500 steps: as the success rate climbs toward~31\%, the model exhausts the reasoning improvements accessible via policy gradient on correct/incorrect completions alone, and further training yields diminishing returns. This motivates a qualitatively different approach.

\subsection{Stage 3: On-Policy Self-Distillation}

The most effective stage is conceptually the simplest. We sample CoT completions from the current policy on a collection of reasoning tasks, score them using the same advantage computation as GRPO (without standard-deviation normalization), and retain the traces with positive advantages. The model is then fine-tuned on its own successful reasoning traces via supervised loss, where the cross-entropy gradient for each example is \emph{scaled by its advantage}. Because the advantages are unnormalized ($A_i = r_i - \bar{r}$), problems that the model rarely solves correctly produce comparatively large advantages for their few correct traces, while problems that are already easy yield small advantages. The supervised loss therefore automatically up-weights hard problems and down-weights mastered ones, providing a built-in curriculum effect without any explicit difficulty scheduling.

This on-policy self-distillation avoids the gradient-balance issues of online RL entirely: the model learns only from completions it has already produced and that have been verified as correct. This approach is closely related to two lines of prior work. STaR~\cite{zelikman2022star} bootstraps reasoning by generating rationales, filtering to those that produce correct answers, and fine-tuning on them iteratively---the key mechanism is identical to our Stage~3. More recently, Liu et al.~\cite{liu2025ssd} propose \emph{Simple Self-Distillation} (SSD), which samples model outputs at various temperatures and fine-tunes on them with standard SFT, without any verifier or reward model. SSD differs from our approach in that it fine-tunes on \emph{all} sampled outputs rather than filtering by correctness; their theoretical analysis traces the gains to a reshaping of token-level distributions that suppresses distractor tails where precision matters while preserving diversity where exploration is beneficial. Our approach can be seen as a scored variant of SSD: by retaining only positive-advantage completions, we perform an explicit quality filter that SSD omits, at the cost of requiring a reward signal. Despite its simplicity, this stage yielded large improvements on reasoning benchmarks in our pipeline. We do not draw a general conclusion from this single data point; the relative effectiveness of each stage is likely sensitive to the specific model, data mix, and starting checkpoint.

\subsection{Results}

Table~\ref{tab:rl_gsm8k} summarizes GSM8K pass@1 accuracy after each stage of the post-training pipeline.

\begin{table}[h]
\centering
\begin{tabular}{lc}
\toprule
\textbf{Stage} & \textbf{GSM8K (\%)} \\
\midrule
After SFT (baseline) & 27.98 \\
+ Dr.\ GRPO (Stage~1) & 28.96 \\
+ Gradient-balanced GRPO (Stage~2) & 31.36 \\
+ Scored self-distillation (Stage~3) & \textbf{33.43} \\
\bottomrule
\end{tabular}
\caption{GSM8K pass@1 accuracy of Nautile-370M after each reinforcement learning stage. Each row is cumulative: Stage~$n$ builds on the checkpoint produced by Stage~$n-1$.}
\label{tab:rl_gsm8k}
\end{table}

\section{Evaluation}

Table~\ref{tab:benchmarks} compares Nautile-370M with publicly available models of similar size across a range of reasoning and language understanding benchmarks. Bold indicates the best score in each column; underline indicates the second best.

\begin{table}[h]
\centering
\small
\begin{tabular}{l ccccc}
\toprule
\textbf{Benchmark} & \textbf{Nautile} & \textbf{Qwen2.5} & \textbf{Granite} & \textbf{LFM2.5} & \textbf{SmolLM2} \\
 & \textbf{370M} & \textbf{0.5B} & \textbf{350M} & \textbf{350M} & \textbf{360M} \\
\midrule
Training tokens & $\sim$0.8T & 18T & 10--12T & 28T & 4T \\
\midrule
OpenBookQA      & \textbf{49.3} & \underline{34.4} & 31.6 & 26.4 & 24.2 \\
ARC             & \textbf{57.0} & \underline{50.1} & 30.0 & 32.8 & 43.7 \\
CommonsenseQA   & \textbf{46.8} & \underline{46.5} & 36.2 & 44.3 & 18.4 \\
GSM8K (0-shot)  & \textbf{33.4} & 28.3 & 31.5 & \underline{33.0} & 7.4 \\
PIQA            & \textbf{61.5} & \underline{61.3} & 50.8 & 49.5 & 48.2 \\
IFEval          & 36.9 & 31.6 & \underline{55.4} & \textbf{62.4} & 41.0 \\
TriviaQA        & 23.8 & \underline{27.8} & 25.2 & 22.9 & \textbf{28.0} \\
MATH500         & 2.4 & \textbf{18.8} & 5.6 & \underline{12.2} & 0.0 \\
MMLU-Pro        & \underline{14.9} & 14.3 & 11.2 & \textbf{18.6} & 10.3 \\
MMLU            & \textbf{39.2} & 33.7 & 35.0 & \underline{39.1} & 35.8 \\
GPQA Diamond    & \textbf{27.3} & 10.1 & \underline{26.3} & 24.8 & 23.2 \\
\midrule
\textbf{Average} & \textbf{35.7} & 32.4 & 30.8 & \underline{33.3} & 25.5 \\
\bottomrule
\end{tabular}
\caption{Benchmark comparison of Nautile-370M against models of similar size. All scores are accuracy (\%), evaluated in 0-shot. The evaluation is strict: if a model produces multiple candidate answers, the response is scored as incorrect. \textbf{Bold} = best; \underline{underline} = second best.}
\label{tab:benchmarks}
\end{table}

\section{Discussion}

Three observations from the development of Nautile-370M are worth highlighting.

\paragraph{Standard GRPO and heavily SFT-trained models.}
Our experience suggests that standard GRPO can be counterproductive when applied to a model whose reasoning is already deeply integrated through extensive supervised pretraining. In this regime, incorrect completions are not random---they are structurally sound reasoning limited by knowledge gaps---and the binary correct/incorrect signal of GRPO suppresses them indiscriminately. The gradient-balanced variant (Section~4.2) and especially on-policy self-distillation (Section~4.3) proved far more effective because they preserve or reinforce the existing reasoning structure rather than penalizing it. We do not claim this is a general failure mode of GRPO for small models; it may be specific to checkpoints with a high ratio of chain-of-thought SFT data.

\paragraph{Self-distillation as an unexpected boost.}
Fine-tuning the model on its own verified correct traces produced a noticeable accuracy gain that we did not anticipate. We tentatively attribute this to the model being exposed to self-consistent, on-distribution correct reasoning traces, but the mechanism is not fully understood and warrants further investigation.

\paragraph{Hybrid layers and the 2:1 ratio.}
The SCA/SCA/Transformer motif was chosen empirically. Whether the optimal ratio changes with model scale, context length, or task distribution is an open question that would benefit from systematic ablation.

\paragraph{Intended use and limitations.}
Nautile-370M is designed for language understanding and scientific reasoning, not for multi-turn conversation or code generation. We consider that open-ended chat requires, at minimum, several billion parameters to maintain coherent persona, long dialogue context, and stylistic flexibility. Similarly, at 371M parameters the model cannot be a useful code agent; we therefore deliberately exclude dedicated coding datasets, coding benchmarks, and code-completion objectives from all training stages. Code may still appear incidentally in the pretraining corpus, but no capacity is spent on it. Our position is that training on code at this scale would dilute the representation budget available for language and reasoning at the cost of both.

Instead, our primary objective is a compact model with solid common-sense reasoning and reliable logic---qualities that make it well suited for \emph{downstream classification and characterization tasks} such as sentiment analysis, intent detection, topic labeling, and structured information extraction. The reasoning-oriented training pipeline is specifically intended to produce a model that can be efficiently fine-tuned for such tasks, where precise understanding matters more than fluent generation at scale. A further target application is large-scale opinion modeling: by instantiating thousands of distinct persona through lightweight conditioning, the model can generate survey-scale opinion responses at high throughput on modest hardware, enabling synthetic population studies that would be prohibitively slow with multi-billion-parameter models.

\section{Conclusion}

We introduced Nautile-370M, a 371M-parameter reasoning-oriented language model combining SCA---a layer grounded in the derivative of the characteristic function---with periodic transformer blocks. We proved that the SCA readout can exactly retrieve any individual prefix token, recover the full weighted distribution, and reproduce any softmax attention output as a special case; moreover, the derivative formulation is structurally necessary, since the characteristic function itself does not support direct value retrieval. The model was trained entirely on a single TPU v4-64 node via the Google TRC program, with reinforcement learning completed on a single DGX Spark. On the RL side, we described an instability of standard GRPO that we encountered with this specific model and proposed two mitigations: gradient-balanced GRPO and on-policy self-distillation, the latter yielding the largest gains in our pipeline.

\bibliographystyle{plainnat}
\bibliography{references}

\end{document}